\pdfoutput=1

\documentclass[11pt]{article}

\usepackage[preprint]{acl}

\usepackage{times}
\usepackage{latexsym}

\usepackage[T1]{fontenc}

\usepackage[utf8]{inputenc}

\usepackage{microtype}
\usepackage{inconsolata}
\usepackage{graphicx}

\DeclareMathSymbol{:}{\mathord}{operators}{"3A}

\title{ProcrustesGPT: Compressing LLMs with Structured Matrices and Orthogonal Transformations
}

\author{Ekaterina Grishina,  Mikhail Gorbunov, Maxim Rakhuba \\
 HSE University \\
  \small{
    Correspondence: \texttt{er.grishina@yandex.ru} 
    }
  }

\usepackage{amsmath}
\usepackage{amssymb}
\usepackage{mathtools}
\usepackage{amsthm}
\usepackage{multirow}
\usepackage{algorithm}
\usepackage{algpseudocode}
\usepackage{booktabs}

\theoremstyle{plain}
\newtheorem{theorem}{Theorem}[section]
\newtheorem{proposition}[theorem]{Proposition}

\theoremstyle{definition}
\newtheorem{definition}[theorem]{Definition}

\theoremstyle{remark}

\DeclareMathOperator*{\argmin}{arg \, min}

\DeclareMathOperator*{\Q}{\mathcal{Q}}

\DeclareMathOperator*{\Si}{\mathcal{S}_{in}}
\DeclareMathOperator*{\So}{\mathcal{S}_{out}}
\DeclareMathOperator*{\Ss}{\mathcal{S}_{skip}}

\begin{document}
\maketitle
\begin{abstract}

Large language models (LLMs) demonstrate impressive results in natural language processing tasks but require a significant amount of computational and memory resources. Structured matrix representations are a promising way for reducing the number of parameters of these models. However, it seems unrealistic to expect that weight matrices of pretrained models can be accurately represented by structured matrices without any fine-tuning.
To overcome this issue, we utilize the fact that LLM output is invariant under certain orthogonal transformations of weight matrices.
This insight can be leveraged to identify transformations that significantly improve the compressibility of weights within structured classes.
The proposed approach is applicable to various types of structured matrices that support efficient projection operations. Code is available at: \url{https://github.com/GrishKate/ProcrustesGPT}.
\end{abstract}

\section{Introduction}
Large language models have achieved remarkable success in language processing and are widely used in a variety of applications, but their deployment is still challenging, as these models hardly fit into a single GPU and require much computational resources during the inference and training processes. The research community is actively seeking efficient algorithms to reduce the size of pretrained models without sacrificing accuracy. 

One approach that has not been fully explored is the use of structured matrices, which can potentially not only reduce the number of parameters in the model but also speed up computations. Low-parametric matrix decomposition can be applied directly to the weight matrices, minimizing the difference between the original and decomposed weights, i.e., $\|W - W'\| \to \min_{W'\in\mathcal{S}}$, where $\mathcal{S}$ is the low-parametric matrix class. Such a factorization may appear to be a reasonable initial guess for further fine-tuning.
Unfortunately, it often leads to high approximation errors for any $\mathcal{S}$ when no additional training is performed, as there are no restrictions on the structure of the weight matrices during pretraining.

To overcome this issue, we utilize invariance of the network output under certain orthogonal transformations, which was first observed in~\cite{ashkboos2024slicegpt}.
In particular, for each layer, we aim to find such transformations that lead to the best compressibility within the chosen matrix class $\mathcal{S}$.
In our paper, we focus on the sum of Kronecker product representation and $\mathcal{GS}$-matrices~\cite{group_and_shuffle,dao2022monarch}, but other representations are also possible.
Finding an optimal orthogonal transformation is a known linear algebra problem and is called the \emph{orthogonal Procrustes problem}.
The resulting framework is formulated as an optimization problem on the weights of the pretrained network, is free from the need for fine-tuning and applicable for different structured matrix representations.
In Figure~\ref{fig:kron_vs_gs}, we present relative errors in the Frobenius norm for different layers with and without using orthogonal transformations.
We observe a noticeable increase in compression ability thanks to optimally chosen orthogonal transformations.

\begin{figure*}[t]
  \includegraphics[width=1.0\linewidth]{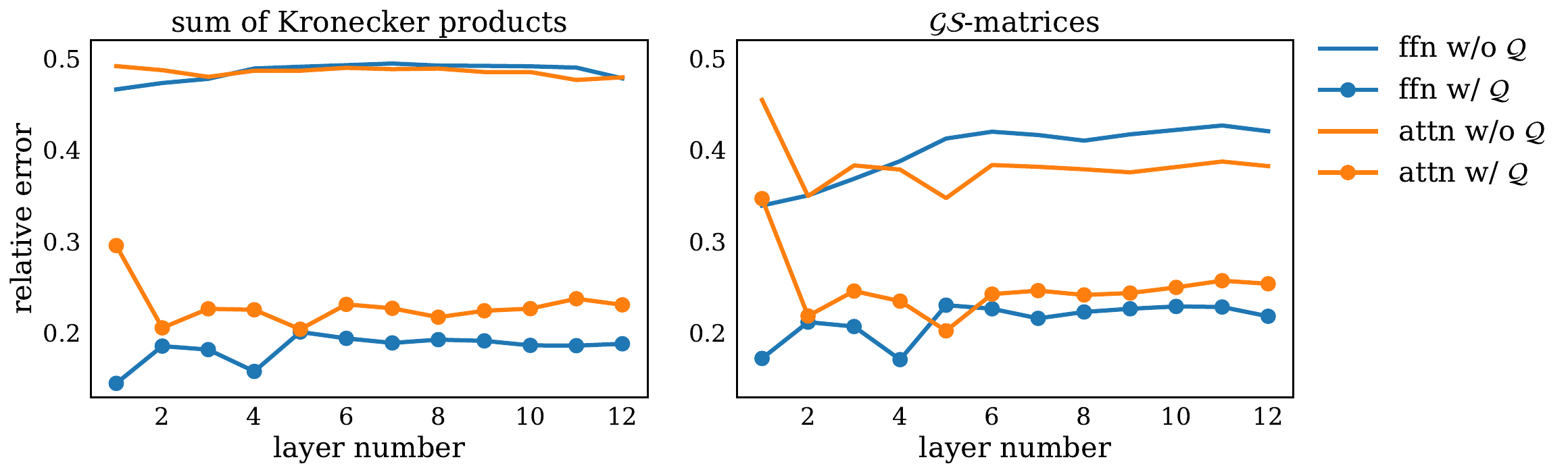} 
  \caption { Illustration of compressibility of different layers of OPT-125m with and without applying orthogonal transformations $\mathcal{Q}$. We consider two types of structured representations: sum of Kronecker products (Section~\ref{sec:kron}) and $\mathcal{GS}$-matrices (Section~\ref{sec:gs}). Both representations result in approximately 25\% compression for each compressed matrix within the layers.} \label{fig:kron_vs_gs}
\end{figure*}

The main contributions of our work:
\begin{itemize}
\item We propose a new fine-tuning free framework for compressing LLM weights. The framework utilizes orthogonal transformations of the network weight matrices to ensure their compressibility using a low-parametric representation of choice.
\item We formulate this framework as an optimization problem and develop efficient numerical methods to solve it. 
\item 
We conduct experiments with OPT \cite{zhang2022opt} and Llama2 \cite{touvron2023llama} models. We show that in most scenarios, our approach yields more accurate results than alternative fine-tuning-free methods at comparable compression rates (in the range from 14\% to 36\%).
\end{itemize}

\section{Related work}

Existing approaches to neural network compression can be divided into four categories: quantization, distillation, pruning, and matrix factorization. Matrix factorization is a promising and relatively underexplored technique, and the most widely used approach within this category is a low-rank approximation. The work \cite{sharma2023laser} has shown that by carefully selecting ranks for each of the weights, the performance of LLMs can be improved on some tasks.
However, direct application of SVD to uniform compression  of weights leads to poor performance because weights usually have high ranks \cite{yu2023features_low_rank}. Instead, the authors \cite{wang2024svd, hsu2022fwsvd, yuan2023asvd, yu2023features_low_rank, ji2024bayesian_svd, chen2021drone} use calibration datasets or Fisher information matrix to approximate activations. Other works explore low-rank plus sparse \cite{li2023losparse} and low-rank plus quantized decompositions \cite{saha2024caldera}.

The works \cite{tahaei2021kronecker_bert,edalati2021kronecker_gpt} have been among the first to apply the Kronecker-product decomposition for the compression of BERT and GPT2. The paper \cite{abronin2024tqcompressor} proposes to enhance the performance of the Kronecker decomposition by adding permutations. These works compress the weight matrices directly and require knowledge distillation or fine-tuning in order to recover performance.

The authors of ModeGPT \cite{lin2024modegpt} propose to split the weight matrices of the transformer into pairs and jointly compress them using SVD, CR or Nystr\"om approximation. This method is training-free and efficiently preserves model performance; however, it does not address compression of embedding and head matrices. 

Many studies have investigated embedding compression, but most of the proposed algorithms require additional training. 
The works \cite{xu2023tensorgpt, hrinchuk2020tensorized} apply the Tensor Train decomposition to the embedding layer, which offers strong compression, but requires training from scratch or hinders model performance. Other methods \cite{lioutas2020distilled, acharya2019online} approximate embedding with low-rank decomposition, but these methods require fine-tuning. GroupReduce \cite{chen2018groupreduce} utilizes knowledge about words occurrence by weighing tokens with their frequencies and applies block low-rank approximation.

Our goal in this study is to maintain model performance after weight factorization without additional training. The authors of the prunning approach SliceGPT \cite{ashkboos2024slicegpt} have introduced the concept of computational invariance, meaning that orthogonal transformations can be applied to transformer weights without changing the output of the model. SliceGPT uses invariance to project layer activations onto their principal components and remove columns or rows from weight matrices. In this work, we utilize computational invariance to find a better approximation of the weights with structured representations.

\section{Our approach}

Our approach consists in iteratively finding orthogonal transformations $\Q$ to obtain optimal compression properties of the matrix weights.
For brevity, we call these orthogonal transformations ``rotations'', although they do not necessarily have determinants equal to one.
In Section~\ref{sec:rot}, we present the concept of rotational invariance.
Then, in Section~\ref{sec:opt}, we present our approach as an optimization formulation to be solved.
Notably, SliceGPT~\cite{ashkboos2024slicegpt} appears to be a particular instance of this formulation, which we also discuss further in Section~\ref{sec:slicegpt}.

\subsection{Rotational invariance}\label{sec:rot}

\begin{figure*}[t]
  \includegraphics[width=1.0\linewidth]{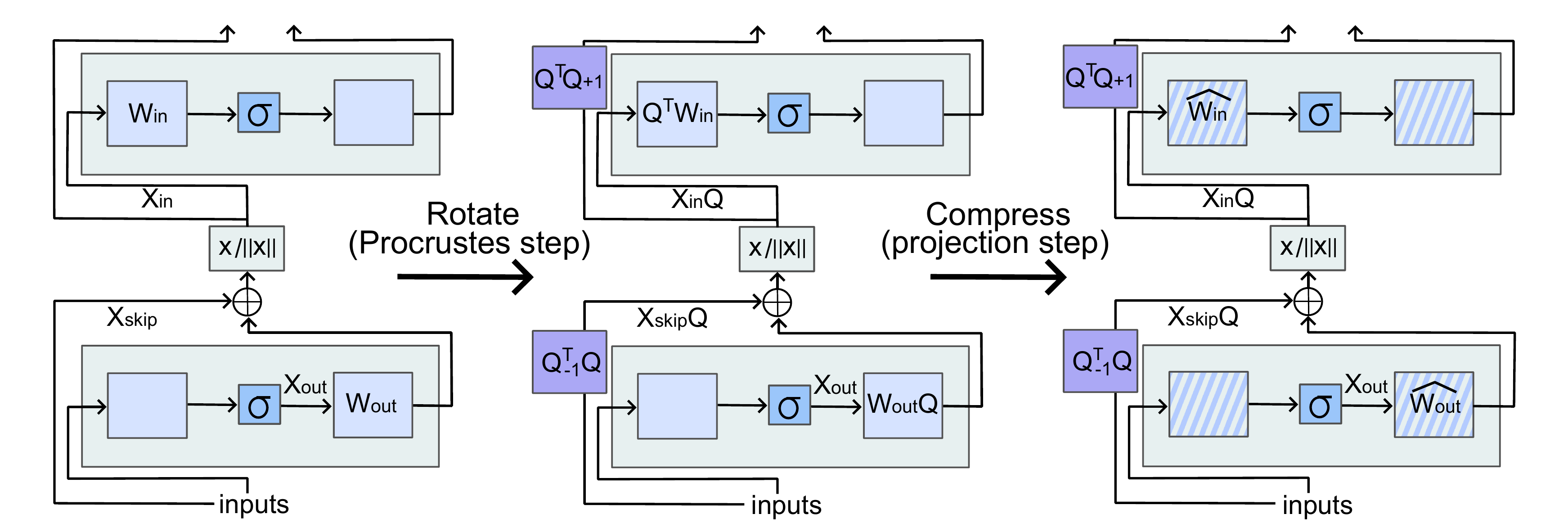} 
  \caption { Illustration of the process of compression of a single transformer layer.} \label{fig:compression}
\end{figure*}

In this section, we introduce notation and explain the concept of rotational invariance~\cite{ashkboos2024slicegpt}. The transformer architecture consists of the repeated series of multi-head self-attention (MHSA) blocks and feed-forward network (FFN) blocks. Between these blocks, there is the LayerNorm or RMSNorm block. The RMSNorm normalizes the input vector: 
\[
    \mathrm{RMSNorm}(x) = \frac{x}{\|x\|_2}. 
\]
LayerNorm is a linear transformation of RMSNorm and a network with LayerNorm can be easily transformed to a network with RMSNorm~\cite{ashkboos2024slicegpt}. 

Each of the MHSA and FFN blocks consists of input linear mappings, an activation function, and an output linear mapping. For example, the MHSA block first obtains queries, keys and values ($XW_q, XW_k, XW_v$) through a linear projection of the input $X$. Then these matrices are nonlinearly transformed and multiplied by an output weight matrix ($W_o$). If we denote the stacked weight matrices of the input linear mappings of each block as $W_{in}$ and the weight matrix of the output linear mapping as $W_{out}$, then we can write:
\[\begin{aligned}\text{MHSA}(X)&=\sigma(X[W_q, W_k, W_v])W_{o}+b=\\ &=\sigma(XW_{in})W_{out}+b,\end{aligned}\]
where $\sigma$ is multi-head attention operation. Similarly, the FFN block can be written as
\[
\text{FFN}(X)=\sigma(XW_{in}+b_{in})W_{out}+b_{out},
\]
where $\sigma$ denotes an element-wise activation function, e.g. ReLU.

Let $Q$ be an orthogonal matrix, which is a square matrix satisfying $Q^T Q = I$. It is well known that the Frobenius norm of a matrix $\|X\|_F^2 = \sum_{i,j} x_{ij}^2$ is invariant under orthogonal transformations of $X$, i.e., for any orthogonal $Q_1$ and $Q_2$:
\begin{equation}\label{eq:unitary-invariance}
    \|Q_1 X Q_2 \|_F = \|X\|_F.
\end{equation}
Using this invariance and $Q^T Q = I$, we may write
\begin{equation}\label{eq:rmsnorm}
\begin{aligned}& 
\left(\frac{XW_{out} + X_{skip}}{\|XW_{out}+X_{skip}\|_F}\right)W_{in} = \\&=
\left(\frac{XW_{out}Q+X_{skip}Q}{\|XW_{out}Q+X_{skip}Q\|_F}\right)(Q^TW_{in}),
\end{aligned}
\end{equation}
where $X_{skip}$ comes from the skip connection. As a result, we modify the skip connections to apply $Q$ to the input of the RMSNorm ($X_{skip}$) and $Q^T$ to the output of RMSNorm (the part of \eqref{eq:rmsnorm} in brackets) to keep the  model unchanged.

\subsection{Optimization problem formulation} \label{sec:opt}

We aim to improve the compression of the model weights via structured matrices by utilizing rotational invariance. 
In essence, for a given structure of layers, we want to find the rotations that complement well with the chosen structure, then rotate the network and project rotated weights on the structured matrix layers, with as little degradation of performance as possible. 
For each layer, our objective is thus minimizing the L2-difference between the outputs of the rotated network and the outputs of the compressed network on the calibration dataset.

For the ease of presentation and prior to delving into a detailed motivation, let us first articulate the final formulation of the optimization problem for the $\ell$-th layer:
\begin{flalign}
\label{eq:optimization_problem}
    & \|X_{out}^{\ell} (W_{out}^{\ell} Q_{\ell} - \widehat{W}_{out}^{\ell}) \|_F^2 + &&\\ 
    & +\nonumber \lambda_{in}\| X_{in}^{\ell} (W_{in}^{\ell} - Q_{\ell} \widehat{W}_{in}^{\ell}) \|_F^2 \to 
    \min_{\substack{
        Q_\ell^TQ_\ell = I, \widehat{W}_{\alpha}^{\ell} \in \mathcal{S}_{\alpha}, \\ \alpha\in\{in,\, out\},
        }}
\end{flalign}
where $\mathcal{S}_{in}$ and $\mathcal{S}_{out}$ are structured matrix classes that are utilized for compression.
Despite the seeming simplicity, it is a nontrivial nonconvex optimization problem.
We propose to approach it by alternatingly optimizing between $Q_{\ell}$ and $\widehat{W}_{in}^{\ell}, \widehat{W}_{out}^{\ell}$. 
Individual optimization problems are respectively a Procrustes problem for finding optimal orthogonal matrix and a projection step in a weighted norm for low-parametric representations. 
Importantly, problems for each layer are independent from each other and can be solved in parallel. 
We further discuss how we tackle this optimization problem in Section~\ref{sec:optimization}.

\subsection{Motivating the optimization formulation}

Let us discuss the motivation for~\eqref{eq:optimization_problem} with a more general approach.
Besides giving us motivation, it also uncovers the connection to the pruning approach of~\cite{ashkboos2024slicegpt}. As discussed in Section \ref{sec:rot}, by utilizing rotation invariance, one can transform different layers of the network with different rotations without changing the network outputs. However,  orthogonal matrices arise in skip connections after the ``rotation'' step (see Figure \ref{fig:compression}), and it can be fruitful for future work to additionally compress them.

We denote the weights of the network with the applied set of rotations $\mathcal{Q}=\{Q_1, Q_2, \dots , Q_L\}$ as $W_{in}^{\Q, \ell}, W_{out}^{\Q, \ell}$ and the intermediate corresponding inputs as $X_{in}^{\Q, \ell}$, $X_{out}^{\Q, \ell}, X^{\Q, \ell}_{skip}$ (see Figure \ref{fig:compression}). Let us denote the input of RMSNorm as
\[
    f^{\Q, \ell}_{out}(X_{skip}^{\Q, \ell}) = X_{out}^{\Q, \ell} W_{out}^{\Q, \ell} + X_{skip}^{\Q, \ell} Q_{\ell - 1}^T Q_{\ell},
\]
and the output of the linear mapping following RMSNorm as
\[
    f^{\Q, \ell}_{in} (X_{in}^{\Q, \ell}) = X_{in}^{\Q, \ell} W_{in}^{\Q, \ell}.
\]
For the $l$-th rotation, objective can be written as
\[
\begin{split}
    & \| f^{\Q, \ell}_{out}(X_{skip}^{\Q, \ell}) - \widehat{f}^{\ell}_{out} (X_{skip}^{\Q, \ell}) \|_F^2 +\\ 
    & +\lambda_{in}\| f^{\Q, \ell}_{in}(X_{in}^{\Q, \ell}) - \widehat{f}_{in}^{\ell}(X_{in}^{\Q, \ell}) \|_F^2 \to \min_{\substack{\widehat{f}_{in}^{\ell}, \widehat{f}_{out}^{\ell},\\ Q_{\ell}^T Q_{\ell} = I}}
\end{split}
\]
Utilization of rotations affects matrices $W^{\ell}_{in}, W^{\ell}_{out}$, while also adding extra matrix $Q_{\ell - 1}^T Q_{\ell}$ in the skip connection. We aim to compress rotated weights $W^{\Q, \ell}_{in}, W^{Q, \ell}_{out}$ while also having the possibility of compressing matrices $Q_{\ell - 1}^T Q_\ell $ into a separate structured matrix $\widehat{W}_{skip}^{\ell}$. Then, our objective becomes:
\[
    \begin{split}
    \| &(X_{out}^{\Q, \ell}  W_{out}^{\Q, \ell} + X_{skip}^{\Q, \ell} Q_{\ell - 1}^T Q_\ell) - \\
    &  (X_{out}^{\Q, \ell} \widehat{W}_{out}^{\ell} + X_{skip}^{\Q, \ell} \widehat{W}_{skip}^{\ell} ) \|_F^2 + \\ 
    &  \lambda_{in} \| X_{in}^{\Q, \ell} W_{in}^{\Q, \ell} - X_{in}^{\Q, \ell} \widehat{W}_{in}^{\ell} \|_F^2 \to 
    \min_{\substack{
        Q_\ell^TQ_\ell = I,\, \widehat{W}_{\alpha}^{\ell} \in \mathcal{S}_{\alpha}, \\ \alpha\in\{in, out, skip\} 
        }}
    \end{split}
\]
 Which can be rewritten as
\[
\begin{split}
    & \|X_{out}^{\ell} (W_{out}^{\ell} Q_{\ell} - \widehat{W}_{out}^{\ell}) +  \\
    &  X_{skip}^{\ell}(Q_{\ell} - Q_{\ell-1} \widehat{W}_{skip}^{\ell} ) \|_F^2 + \\ 
    &  \lambda_{in} \| X_{in}^{\ell} (W_{in}^{\ell} - Q_{\ell} \widehat{W}_{in}^{\ell}) \|_F^2 \to 
    \min_{\substack{
        Q_\ell^TQ_\ell = I,\, \widehat{W}_{\alpha}^{\ell} \in \mathcal{S}_{\alpha}, \\ \alpha\in\{in, out, skip\} 
        }}
\end{split}
\]

We can approximate this further by abandoning the compression of $\widehat{W}_{skip}^{\ell}$, setting it to be equal to $Q_{\ell-1}^T Q_{\ell}$. This way we arrive at~\eqref{eq:optimization_problem}, and optimization problem for the $\ell$-th layer becomes independent of the solution for the previous layers. This allows solving problems for different layers in parallel. Additionally, experiments have shown that balancing the terms in~\eqref{eq:optimization_problem} using $\lambda_{in} = \frac{\|X_{out}W_{out}\|^2}{\|X_{in}W_{in}\|^2}$ improves the quality compared to $\lambda_{in}=1$ (see Table~\ref{table:ppl}).

\section{Structured matrix representations}\label{sec:structall}

In this section, we present different structured matrix representations on which we focus in our work. 
The sum of Kronecker products yields the most consistent accuracy gain within different models.
Although the $\mathcal{GS}$-matrix representation resulted in slightly lower accuracy overall, its computationally efficient structure offers significant potential to accelerate inference.
Finally, we examine matrices with zero blocks, revealing connections with pruning techniques~\cite{ashkboos2024slicegpt}.

\subsection{Kronecker products}\label{sec:kron}

\begin{definition}
Given matrices $A \in \mathbb{R}^{m \times n}$ and $ B \in \mathbb{R}^{p \times q}$, the Kronecker product $A \otimes B$ is the $pm \times qn$ block matrix:
\[A \otimes B = \begin{bmatrix}a_{11}B & \dots & a_{1n}B\\ \vdots & \ddots & \vdots \\ a_{m1}B & \dots & a_{mn}B \end{bmatrix}.
\]
\end{definition}
One Kronecker product is a very restrictive structure, and one usually considers the sum of $r>1$ Kronecker products for more accurate results.
Fortunately, the problem of obtaining the best approximation within such a structure (projection operation):
\begin{equation}\label{eq:kronstruct}
    \left\|W-\sum_{i=1}^r A_i\otimes B_i\right\|^2_F \to \min_{A_i, B_i},
\end{equation}
has an analytical solution that can be obtained using SVD~\cite{golub2013matrix}, see Algorithm~\ref{alg:kron}.
For better results, we need to use the weighted Frobenius norm~\eqref{eq:optimization_problem}.
The Kronecker product approximation problem with the weight matrix~$X$ reads as:
\begin{equation}\label{eq:kronstructweight}
    \left\|X\left(W-\sum_{i=1}^r A_i\otimes B_i\right)\right\|^2_F \to \min_{A_i, B_i}.
\end{equation}
Although it does not admit a simple solution, the SVD-based solution of~\eqref{eq:kronstruct} can be used for initialization for the iterative process.
As an iterative procedure, we optimize~\eqref{eq:kronstructweight} alternatively with respect to $\{A_i\}$ and $\{B_i\}$.
Each of the alternating subproblems is solved exactly, details are presented in Appendix~\ref{sec:kron_weighted}.

\begin{algorithm}
\caption{SVD-based solution to~\eqref{eq:kronstruct}.}\label{alg:kron}
\begin{algorithmic}
\State \textbf{Input:} $W\in \mathbb{R}^{mp\times nq}$, rank $r$.
\State \textbf{Output:} $\{A_i\}$, $\{B_i\}$ from~\eqref{eq:kronstruct}.
\State $W_r = W.\texttt{rearrange((m p) (n q) $\to$ }$
\State \hfill \texttt{(m n) (p q))}$\qquad\qquad$
\State $USV^T=\texttt{SVD}(W_{r})$
\State $A' = U[: , :r]S[:r,  :r]^{1/2}$
\State $B' = S[:r, :r]^{1/2}V[:, :r]^T$
\State $A = A'.\texttt{rearrange('(m p) r $\to$ r m p')}$
\State $B = B'.\texttt{rearrange('(n q) r $\to$ r n q')}$
\State \textbf{return} $A, B$  
\end{algorithmic}
\end{algorithm}

\subsection{\texorpdfstring{$\mathcal{GS}$}-matrices}\label{sec:gs}

\begin{definition}
    $\mathcal{GS}$-matrices are matrices that can be represented in the form $P_L (L P R) P_R$, where $L, R$ are block-diagonal matrices and $P_L, P, P_R$ are permutation matrices. 
\end{definition}
This class of matrices \cite{group_and_shuffle} generalizes Monarch \cite{dao2022monarch} matrices and describes matrices with low-rank blocks up to a permutation of rows and columns. Thanks to this property, the projection step
\[
    \| W - P_L (L P R) P_R \|_F^2 \to \min_{L, R}
\]
can be performed efficiently using an SVD-based procedure described in \cite{group_and_shuffle}. 
Likewise for the Kronecker decomposition, the weighted approximation problem
\[
    \| X(W - P_L (L P R) P_R )\|_F^2 \to \min_{L, R}
\]
does not admit a simple solution.
Nevertheless, it can still be solved numerically by alternating iterations with respect to $L$ and $R$, see Appendix~\ref{sec:gs_weighted}.

\subsection{Matrices with zero blocks and relation to SliceGPT \label{sec:slicegpt}} 
Another structured class one could consider is block-sparse matrices, which include matrices with a single nonzero block. 
For example, this includes matrices of the forms:
\[
    \begin{pmatrix}
        W & 0
    \end{pmatrix}, 
    \quad 
    \begin{pmatrix}
        W \\
        0
    \end{pmatrix},
    \quad
    \begin{pmatrix}
        W & 0 \\
        0 & 0
    \end{pmatrix}.
\]
Such structures are not frequently used due to their simplicity and poor expressivity. 
Nevertheless, when paired with rotations, they can become a useful representation. 
Interestingly enough, we find that utilizing these classes and solving our objective have some relation to the SliceGPT method.  
\label{prop:slicegpt}
\begin{proposition}
    Let $\lambda_{in} = 0$. Let $\Si$ and $\Ss$ be  matrices with zero columns. Then, solving our objective is equivalent to finding the rotation of the SliceGPT method and column slicing of $W_{out}$ and $Q_{\ell - 1}^T Q_{\ell}$. Row slicing of $W_{in}$ and $Q_{\ell - 1}^T Q_{\ell}$ arises naturally due to the sparse structure of inputs. 
\end{proposition}

\begin{proof}
    See Appendix~\ref{apx:prop1_proof}.
\end{proof}

\section{Optimization algorithm} \label{sec:optimization}

\subsection{Orthogonal Procrustes problem \label{sec:opp}}

The problem of finding an orthogonal matrix $Q$, which most closely fits a given matrix $A$ to $B$, is called an orthogonal Procrustes problem (OPP):
\[\|QA-B\|_F\to\min_{Q^TQ=I}.\]
It was first solved in the work \cite{schonemann1966generalized}. The solution is $Q=UV^T$, where $U$ and $V^T$ are from the SVD: $BA^T = U\Sigma V^T$.

The extension of the Procrustes problem, where an orthogonal matrix $Q$ is multiplied from two sides, is called the weighted orthogonal Procrustes problem (WOPP):
\[\|CQA-B\|_F\to\min_{Q^TQ=I}.\]
Unfortunately, WOPP does not have a simple analytical solution \cite{lissitz1976solution}. 
We solve it by parametrizing $Q$ using the so-called Cayley transform~\cite{golub2013matrix} and using conjugate-gradients (see Algorithm~\ref{alg:weighted_als} and~\eqref{eq:cayley}).

\subsection{Efficient initialization \label{sec:initialization}}
As we mentioned above, the optimization problem~\eqref{eq:optimization_problem} does not admit a simple solution. 
Therefore, we suggest finding a proper initialization for the arising matrices  first. 
As a good initial point we can use optimal solution in the Frobenius norm:
\[\|W_{out}Q-\widehat{W}_{out}\|^2_F+\|Q^TW_{in}-\widehat{W}_{in}\|^2_F\to \min_{\substack{\widehat{W}_{out}, \widehat{W}_{in},\\
Q^TQ=I}}\]
To solve this problem, we can rewrite it as:
\[
\left\|[W_{out}, W_{in}^T]Q-[\widehat{W}_{out}, \widehat{W}_{in}^T]\right\|^2_F\to \min_{\substack{\widehat{W}_{out}, \widehat{W}_{in}, \\ Q^TQ=I}},
\]
which is solved using the alternating scheme called alternating least squares (ALS), see Algorithm~\ref{alg:als}.
In particular, we first find optimal  $\widehat{W}_{out}, \widehat{W}_{in}$ for the fixed $Q$, and then optimal orthogonal matrix $Q$ for the fixed $\widehat{W}_{out}, \widehat{W}_{in}$ (orthogonal Procrustes problem).
The process is repeated until the maximum number of iterations is reached.

\begin{algorithm}
\caption{ALS in the Frobenius norm}\label{alg:als}
\begin{algorithmic}
\State \textbf{Input:} $W_{out}, W_{in}$.
\State Set $Q=I$.
\For{$1 \dots \text{n\_iters}$}
    \State $\triangleright$ Projection step (Section~\ref{sec:structall}):
    \State $\quad\widehat{W}_{in} = \argmin\limits_{W \in \Si} \|Q^TW_{in} - W\|^2_F$ 
    \State $\quad\widehat{W}_{out} = \argmin\limits_{W \in \So} \|W_{out}Q - W\|^2_F$ 
    \State $\quad W_{appr} = [\widehat{W}_{out}, \widehat{W}_{in}^T]$
    \State $\quad W = [{W}_{out}, {W}_{in}^T]$
    \State $\triangleright$ Solve OPP, using SVD (Section~\ref{sec:opp}) 
    \State $\quad Q = \argmin\limits_{Q^TQ=I} \|WQ - W_{appr}\|^2_F $ 
\EndFor
\State \textbf{return} $Q, \widehat{W}_{in}, \widehat{W}_{out}$ 
\end{algorithmic}
\end{algorithm}

\subsection{Alternating iteration}
After obtaining initializations for orthogonal layers from ALS in Frobenius norm, we proceed with more computationally challenging weighted optimization scheme, described in Algorithm~\ref{alg:weighted_als}. Even though we only do a handful of steps of the algorithm, this step is crucial and noticeably improves results, as is shown in Table \ref{table:ppl}.  
\begin{algorithm}[tbh]
\caption{ALS in the weighted norm}\label{alg:weighted_als}
\begin{algorithmic}
\State Before compression in weighted norm, rotate the network with $Q_{init}$ from Algorithm~\ref{alg:als}.
\State \textbf{Input} $W_{out}, W_{in}, X_{out}, X_{in}$.
\State Set $Q = I$
\For{$1 \dots \text{n\_iters}$} 
    \State $\triangleright$ Weighted norm projection (Appx. \ref{sec:kron_weighted}, \ref{sec:gs_weighted}):
    \State $\quad\widehat{W}_{out}=\argmin\limits_{W \in \Si}\|X_{out}(W_{out}Q -W)\|^2_F$ 
    \State $\quad\widehat{W}_{in} = \argmin\limits_{W \in \So}\|X_{in}(W_{in} -QW)\|^2_F$ 
    \State $\triangleright$ Solve WOPP by parametrizing $Q$ with Cayley transform~\eqref{eq:cayley} and using conjugate gradients:
    \State $\quad \begin{aligned}Q = \argmin_{Q^TQ=I} &\|X_{out}(W_{out}Q- \widehat{W}_{out})\|^2_F + \\ + &\|X_{in}(W_{in}- Q\widehat{W}_{in})\|^2_F\end{aligned}$ 
\EndFor
\State \textbf{return} $Q, \widehat{W}_{in}, \widehat{W}_{out}$ 
\end{algorithmic}
\end{algorithm}

\subsection{Practical aspects}

\textbf{Computing inputs.} The input of the linear layer~$X$ is a matrix of the shape $bs\times n$, where $b, s$ and $n$ are respectively the sequence length, the number of calibration samples and the hidden dimension. 
Typically, $b$ and $s$ are large, so making computations with $X$ is challenging. 
However, we can use square root of the smaller $n\times n$ correlation matrix $X^TX \in\mathbb{R}^{n\times n}$ instead of $X$ to solve the optimization problem. Indeed, 
\[
    \|(X^TX)^{1/2}(\dots)\|_F=\|X(\dots)\|_F.
\]
To efficiently compute the correlation matrix, we divide $X\in\mathbb{R}^{bs\times n}$ into smaller matrices (batches) $X_i\in\mathbb{R}^{b_is\times n}$, which fit into memory, and compute 
$
    X^TX=\sum_i X_i^T X_i.
$

\textbf{Embedding and head layers.} For the embedding layer the inputs $X_{in}$ are one-hot vectors with ones standing in the position of the chosen tokens. Therefore, the correlation matrix $X^{T}_{in}X_{in}$ is equal to the diagonal matrix $D$, where $D_{ii}$ is the number of times the $i$-th token appears in the calibration dataset. This simplifies the weighted problem to:
\[
    \| \sqrt{D} (W_{emb} Q - \widehat{W}_{emb}) \|_F^2 + \|\dots \|_F^2\to \min_{Q},
\]
where $W_{emb}$ is the weight of embedding, $\widehat{W}_{emb}$ is its approximation with the matrix decomposition.
This prompted us to also experiment with different functions of $D$ and we found that $\sqrt{D+1}$ gave the best results, although $\log(D+1)$ also worked well (see Tables \ref{table:ppl}, \ref{table:zero_shot}).
We have also discovered that weighting the model's head with the same diagonal matrix as embedding during compression in Frobenius norm additionally improves the performance. 

\textbf{Orthogonal parametrization.} When rotating the network with the set of weights $\Q = \{Q_1, Q_2, \dots Q_{L}\}$, one downside is additional weights $Q_{\ell- 1}^TQ_{\ell}$ that arise in skip connections and should also be stored as the weights. 
This negatively affects the compression ratio. 
One trick to reduce the number of parameters is an observation that $Q_{\ell- 1}^TQ_{\ell}$ is an orthogonal matrix. 

Orthogonal matrices (except for those that have an eigenvalue exactly equal to $-1$) of size $d \times d$ can be represented through the 
Cayley transform:
\begin{equation} \label{eq:cayley}
    Q = (I + K)(I - K)^{-1},
\end{equation}
or through the matrix exponential (if $\det Q=1$):
\begin{equation} \label{eq:exponent}
    Q = \sum^{\infty}_{n=0}\frac{K^n}{n!}
\end{equation}
where $K$ is skew-symmetric: $K = -K^T$. 
This allows us to only store the upper triangular part of the matrix $K$, which reduces the number of parameters from $d^2$ to $\frac{d(d-1)}{2}$. If $\det Q = -1$, any row of $Q$ can be multiplied by $-1$ to change $\det$ to $1$. For the Cayley transform, if there exist eigenvectors $u_i$ corresponding to $-1$ eigenvalues, we multiply $Q$ by Householder matrices $I-2vv^T/\|v\|_2^2$, where $v=Re(u_i)$ or $v=Im(u_i)$, to eliminate all $-1$ from the spectrum. 

\section{Experiments}

\begin{table*}[t]
\begin{center}
\resizebox{\textwidth}{!}{%
\begin{tabular}{c|ccc|cccccc|cccc}
\toprule
\multirow{3}{*}{Method}  &  \multirow{3}{*}{Norm} &  \multirow{3}{*}{Struct.} &  \multirow{3}{*}{Coef.} & \multicolumn{6}{c}{OPT} & \multicolumn{4}{c}{Llama2} \\
& & & &  \multicolumn{2}{c}{125m} &\multicolumn{2}{c}{2.7b}& \multicolumn{2}{c}{13b} & \multicolumn{2}{c}{7b} &\multicolumn{2}{c}{13b}\\
& & & & ppl & \%  &ppl & \% &ppl & \% & ppl & \% &ppl & \%\\
\midrule
Dense  &  & & &27.65 & 0 & 12.47 & 0 & 10.13 & 0 &5.47 & 0 &4.88& 0\\
SliceGPT &  & &  &38.65 & 20.12&14.84 & 16.51&11.12 & 16.00& 7.60 &16.04& 6.60 & 15.94\\
\hline 
ProcrustesGPT & F & Kron. & $\log(D+1)$ & 58.17& 19.59 & 82.47 & 15.57 & 15.20 & 15.00 &16.37 & 15.04  & 8.43 & 14.93\\
ProcrustesGPT  &  W & Kron. & $\log(D+1)$ & 38.48 &19.59& 14.17 & 15.57 & 10.87 &  15.00& 8.48 & 15.04 & 5.72 & 14.93\\

\hline 

ProcrustesGPT & F & Kron. &  $\sqrt{D+1}$ & 55.91 & 19.59 & 78.97 & 15.57 & 15.38 & 15.00 & 11.98 & 15.04  & 8.39  & 14.93\\
ProcrustesGPT  &  W & Kron. &  $\sqrt{D+1}$ & \textbf{36.08} &19.59&  13.95 & 15.57 & \textbf{10.67} &  15.00 & \textbf{6.54} & 15.04 & \textbf{5.71} & 14.93\\

\hline

ProcrustesGPT & F & GS &  $\sqrt{D+1}$ & 154.09 & 19.45 & 152.51 & 15.57 & 261.57 & 15.00 & 11.65 & 15.05 & 8.80 & 14.93 \\
ProcrustesGPT  &  W & GS &  $\sqrt{D+1}$ & 39.58 & 19.45 & \textbf{13.81} & 15.57 & 10.68 & 15.00 & 6.76 & 15.05 & 5.96 & 14.93\\
\bottomrule
\end{tabular}
}
\end{center}
\caption{\label{table:ppl} Perplexity results on WikiText2. The calibration dataset size is 128 sequences of 2048 tokens. ``Ppl'' denotes perplexity, the columns with $\%$ show the percentage of parameters compressed. ``F'' stands for the optimization in Frobenius norm, which is used before optimization in weighted norm, denoted as ``W''.``Coef.'' stands for the diagonal matrix, which weighs embedding and head. In the rows with $\log(D+1)$, $\lambda_{in}=1$; in the rows with $\sqrt{D+1}$, $\lambda_{in} = \frac{\|X_{out}W_{out}\|^2}{\|X_{in}W_{in}\|^2}$. }
\end{table*}

\begin{table*}
\begin{center}
\resizebox{\textwidth}{!}{%
\begin{tabular}{c|c|cc|ccccc|c}
\toprule
Model & Method  & Struct. & Coef. &ARC-c & ARC-e &  HellaS. & PIQA & WinoG. & Average \\
\midrule
\multirow{6}{*}{OPT-13b} & Dense & & & 35.75 & 61.83 & 69.88 & 76.82 &	65.19 & 61.89 \\
& SliceGPT & & &33.62 & 61.95 & 62.99 & 73.67 & 63.30 & 59.10 \\
& ProcrustesGPT & Kron. & $\log(D+1)$ & 36.52 & 58.71 & 68.27 & 76.22 & 63.61 & 60.67\\
& ProcrustesGPT & Kron. & $\sqrt{D+1}$ & \textbf{36.60} & \textbf{62.25} & \textbf{68.41} & \textbf{76.44} & 64.96 & \textbf{61.73}\\
& ProcrustesGPT & GS &  $\sqrt{D+1}$ & 35.07 & 59.09 & 67.14 & 76.06 & 65.59 & 60.59\\

\midrule
\multirow{6}{*}{Llama2-7b} & Dense & & &46.25 & 74.58& 75.99 & 79.11 & 68.82 &	68.95\\
& SliceGPT & &&35.15 & 56.10 & 53.04 & 65.78 & 62.98 & 54.61 \\
& ProcrustesGPT & Kron. & $\log(D+1)$ & 41.98 & \textbf{68.35} & 69.72 & 73.94 & \textbf{67.40} & 64.28\\
& ProcrustesGPT & Kron. & $\sqrt{D+1}$ & \textbf{42.32}  & 68.01 & \textbf{70.20} & \textbf{76.39} & 66.30 & \textbf{64.64}\\
& ProcrustesGPT & GS & $\sqrt{D+1}$ & 36.69 & 62.88 & 66.91 & 74.37 & \textbf{67.40} & 61.65\\

\midrule
\multirow{6}{*}{Llama2-13b} & Dense & && 49.23& 77.53 & 79.36 & 80.52 & 72.30 & 71.79 \\
& SliceGPT & && 39.51 & 62.92 & 56.98 & 67.25 & 67.64 & 58.86\\
& ProcrustesGPT & Kron. & $\log(D+1)$ &\textbf{44.97} & 73.19 & 73.43 & \textbf{77.58} & 70.48 & 67.93\\
& ProcrustesGPT & Kron. & $\sqrt{D+1}$ & 44.80 & 73.23 & \textbf{73.94} & 77.53 & \textbf{70.88} & \textbf{68.07}\\
& ProcrustesGPT & GS & $\sqrt{D+1}$ & 44.45 & \textbf{73.32} & 71.19 & 76.50 &  70.32 & 67.16 \\
\bottomrule
\end{tabular}
}
\end{center}
\caption{\label{table:zero_shot} Zero-shot task performance of compressed OPT-13b, Llama2-7b and Llama2-13b. ProcrustesGPT compresses the weights with Kronecker product. The compression ratio is the same as in Table \ref{table:ppl}.}
\end{table*}

\subsection{Setup}
 We implement our method for OPT  \cite{zhang2022opt} and Llama2 \cite{touvron2023llama} models using Hugging Face Transformers \cite{wolf2020transformers}. As the calibration data we use 128 sequences of length 2048 from  WiKiText2 dataset \cite{merity2016pointer}.
 Experiments were run on a single V100 GPU with 32GB memory or A100 GPU with 80GB memory.
 We evaluate zero-shot performance using LM Evaluation Harness \cite{eval-harness}  across ARC-e, ARC-c \cite{clark2018think}, PIQA \cite{bisk2020piqa},
WinoGrande \cite{sakaguchi2021winogrande}, and HellaSwag \cite{zellers2019hellaswag} datasets.

\subsection{Details of implementation}

Every pair of the blocks $W_{out}, W_{in}$ is rotated independently with its own orthogonal matrix, which allows us to parallelize the computation of the orthogonal matrices. We have implemented parallelization of ALS in Frobenius norm (Algorithm~\ref{alg:als}). The speedup depends on the number of processes that can be run simultaneously in GPU memory.

The compression process consists of two stages. 
Firstly, we compute optimal orthogonal matrices in the Frobenius norm (Algorithm \ref{alg:als}), then we rotate the network using them and run compression in the weighted norm (Algorithm \ref{alg:weighted_als}). We use 50 ALS iterations in Frobenius norm for small models and 25 iterations for 13b models.  We use 1 iteration of ALS in the weighted norm  for all models. We parametrize orthogonal matrix with the Cayley transform and apply 500 iterations of conjugate gradients to find optimal orthogonal matrices in the weighted norm. 
We do not compress the Values matrix at all, as it noticeably degrades the results.

\subsection{Results}

\textbf{Generation performance.} We assess the generation performance of the compressed models using the WikiText2 test set. Table \ref{table:ppl} compares our method against SliceGPT at approximately 25\% compression of the weight matrices. For the details on choice of matrix sizes in decompositions, see Appendix \ref{sec:decomposition_sizes}.  The ``F'' row demonstrates the perplexity achieved by finding optimal orthogonal matrices in the Frobenius norm (see Section \ref{sec:initialization}). We observe that compression using the Frobenius norm alone is not sufficient to preserve model performance,  particularly for small models. However, further compression in the weighted norm helps to regain perplexity, which is shown in the ``W'' row. We observe that $\mathcal{GS}$ performs consistently worse than Kronecker products in the Frobenius norm for OPT models, and produces approximately the same results in the weighted norm. Our method surpasses SliceGPT at a similar level of compression. 
The compression ratios in the tables are presented with respect to the full model size.

Table \ref{table:comparison_ppl} compares ProcrustesGPT to the methods that do not compress embedding and model's head. SVD-LLM \cite{wang2024svd} applies weighted low-rank approximation, while DISP-LLM \cite{gao2024disp} and SLEB \cite{song2024sleb} are pruning methods. ProcrustesGPT outperforms
other baselines at lower compression rates, but its performance starts deteriorating at 36\% compression of parameters. 

\textbf{Zero-shot tasks.} We evaluate our models compressed in the weighted norm on five zero-shot tasks. Our method consistently outperforms SliceGPT, as shown in Table \ref{table:zero_shot}. The difference is more pronounced for the Llama2 models, where our method surpasses SliceGPT by an average of 9-10\%. We notice that compression with sum of Kronecker products better maintains model quality than $\mathcal{GS}$.  Table \ref{table:comparison_zero_shot} shows that on average ProcrustesGPT achieves better zero-shot performance than other baselines.

\begin{table*}[t]
\begin{center}
\resizebox{\textwidth}{!}{%
\begin{tabular}{c|cc|cc|cc|cc|cc|cc}
\toprule
 & \multicolumn{6}{c}{Llama2-7b} & \multicolumn{6}{c}{Llama2-13b} \\
Method  & ppl & \% & ppl & \% & ppl & \% & ppl & \% & ppl & \% & ppl & \% \\ 
\midrule
Dense & 5.47 & 0 & 5.47 & 0  & 5.47 & 0 & 4.88 & 0 & 4.88 & 0 & 4.88 & 0\\
SVD-LLM &   7.86 & 14.44 & 9.73 & 25.00 & 14.39 & 35.58 & 6.34 & 14.64 & 7.53 & 25.36 & 10.08 & 36.09\\
DISP-LLM & 6.80 & 14.31 & 8.52 &25.02 & \textbf{10.92} & 35.60 & 6.23 & 14.60 & 7.90 & 25.36 & \textbf{10.05} & 36.13\\
SLEB & 6.95 & 12.01 & 10.39 & 24.03 & 22.76 & 36.04 & 5.85 & 12.19 & 7.73 & 24.37 & 11.36 & 36.56 \\
ProcrustesGPT (Kron) & \textbf{6.43} & 14.07  & 8.19 & 25.09 & 19.55 & 36.11 & \textbf{5.68} & 14.30 & \textbf{6.95} & 25.48 & 16.88 & 36.66 \\
ProcrustesGPT (GS) & 6.65 & 14.08 & \textbf{7.97} & 25.08 & 14.20 & 36.12 & 5.94 & 14.30 & 7.02 & 25.48 & 10.85 & 36.66\\
\bottomrule
\end{tabular}
}
\end{center}
\caption{\label{table:comparison_ppl} Perplexity of compressed Llama2 on WikiText2. Embedding and head are not compressed for all methods.
The calibration dataset size is 128 sequences of 2048 tokens. $\%$ shows the percentage of parameters compressed. }
\end{table*}

\begin{table*}[t]
\begin{center}
\begin{tabular}{c|c|ccccc|c}
\toprule
 Method & $\%$ &ARC-c & ARC-e &  HellaS. & PIQA & WinoG. & Average \\ 
\midrule
Dense & 0 & 49.23 & 77.53 & 79.36 & 80.52 & 72.30 & 71.79\\
\midrule 
SVD-LLM & 14.64 & 39.25 & 65.61 & 63.92 & 73.83 & 68.35 & 62.19\\
DISP-LLM & 14.60 & \textbf{47.61} & 70.12 & \textbf{74.77} & 76.93 & 69.61 & 67.81 \\
SLEB  & 12.19 & 46.33 & 72.77 & 74.11 & \textbf{78.18} & 69.85 & 68.25 \\
ProcrustesGPT (Kron) & 14.30 & 45.56 & \textbf{74.16} & 74.29 & 77.58 & 70.24 & \textbf{68.37}\\
ProcrustesGPT (GS) & 14.30 & 45.05 & 73.40 & 71.71 & 76.77 & \textbf{70.56} & 67.50 \\
\midrule 
SVD-LLM & 25.36 & 32.76 & 54.34 & 54.19 & 68.12 & 65.98 & 55.08 \\
DISP-LLM & 25.36 & 40.27 & 61.83 & \textbf{69.40} & 73.56 & 63.30 & 61.67\\
SLEB & 24.37& 38.14 & 63.47 & 66.78 & \textbf{76.39} & 60.70 & 61.10\\
ProcrustesGPT (Kron) & 25.48 & 38.48 & 70.03 & 66.42 & 75.03 & 66.06 & 63.20	\\
ProcrustesGPT (GS) &  25.48 & \textbf{41.04} & \textbf{71.38} & 66.12 & 74.71 & \textbf{67.17} & \textbf{64.10}\\
\midrule
SVD-LLM & 36.09 & 26.96 & 43.22 & 43.35 & 61.53 & 60.54 & 47.12
\\
DISP-LLM & 36.13 & 30.72& 53.03 & \textbf{60.65} & 68.72 & 58.64 & 54.35\\
SLEB & 36.56 & 33.62 & 52.15 & 58.42 & \textbf{71.27} & 59.67& 55.03\\
ProcrustesGPT (Kron) & 36.66 & 29.69 &52.95 &48.56 &64.80 &59.98 &51.20 \\
ProcrustesGPT (GS) & 36.66 & \textbf{33.70} & \textbf{61.87} &52.58 &	67.85 & \textbf{62.43} & \textbf{55.69}\\
\bottomrule
\end{tabular}
\end{center}
\caption{\label{table:comparison_zero_shot} Zero-shot performance of compressed Llama2-13b. Embedding and head are not compressed for all methods. \% shows the percentage of model parameters compressed. }
\end{table*}

\section{Conclusion}
This paper presents an approach to LLM compression with structured matrix factorizations, which is suitable for compression with various types of decompositions, including Kronecker products and $\mathcal{GS}$ matrices. Our method maintains performance in generation and zero-shot tasks, and does not require recovery fine-tuning. We hope this work will inspire further research on training-free compression with structured representations.

\section{Limitations}
A natural question may arise if the models compressed using ProcrustesGPT can be fine-tuned with standard methods such as LoRA. 
The structured weights may not align well with the low-rank nature of adapters.
As a result, after the fine-tuning we will need to store the structured representation of initial weights and the LoRA adapters separately, which may be inconvenient.
Therefore, models compressed using structured factorizations require the development of PEFT methods that are better suited to these structures.

\section{Acknowledgments}
Support from the Basic Research Program of HSE University is gratefully acknowledged. The calculations were performed in part through the computational resources of HPC facilities at HSE University~\cite{kostenetskiy2021hpc}.

\appendix

\section{Kronecker approximation in weighted norm \label{sec:kron_weighted}}
\label{sec:appendix}
Let $A\in \mathbb{R}^{r\times m_1 \times n_1}, B\in \mathbb{R}^{r\times m_2 \times n_2}$. The problem of Kronecker approximation in the weighted norm can be written as
\[g(A, B)=\left\|Y-X\sum^{r}_{i}(A_i\otimes B_i)\right\|_F^2\to \min_{A, B}.\]
We can not solve this problem explicitly, so we will iteratively optimize this expression with respect to $A$ for the fixed $B$ and with respect to $B$ for the fixed $A$. We can rewrite it as
\[\begin{split}
& g(A, B)=tr \left( Y^TY-2Y^TX \left( \sum^r_i A_i\otimes B_i \right) +\right. \\ & \left.+ \left( \sum^r_iA_i\otimes B_i \right)^TX^TX \left( \sum^r_j A_j\otimes B_j \right) \right) . \end{split}
\]
Now we take the differential with respect to $A$
\[\begin{split}
&dg=tr \left( 2Y^TX \left( \sum^r_i dA_i\otimes B_i\right) +\right.\\& \left. +2 \left( \sum^r_i A_i\otimes B_i \right)^TX^TX \left( \sum^r_j dA_j\otimes B_j \right) \right) .\end{split}\]
To find the optimal $A$ we should equate the gradient with respect to $A$ to zero. It is complicated to write the solution using formulas, so we will use tensor diagrams instead.
In tensor diagrams, the circles denote multidimensional arrays, the connections between them denote the summation by dimensions. Let us equate the $dg$ to zero and illustrate it as follows:
\[\begin{split}
&tr\left(\left(\sum^r_i A_i\otimes B_i \right)^TX^TX \left( \sum^r_j dA_j\otimes B_j \right) \right) - \\ &-tr \left( Y^TX \left( \sum^r_i dA_i\otimes B_i\right) \right) = 0\end{split}\]
\begin{center}
\resizebox{0.5\textwidth}{!}{%
\includegraphics{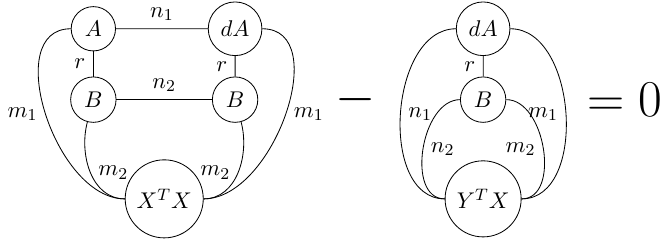}
}%
\end{center}
Now we equate the gradient with respect to $A$ to zero:
\begin{center}
\resizebox{0.5\textwidth}{!}{%
\includegraphics{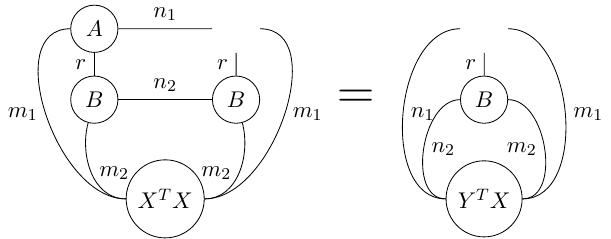}
}%
\end{center}
We can reshape the tensors in the left and right parts into matrices. Let us denote the matrix in the right part by $D$. The left part is the tensor $A$ reshaped into matrix of size $rm_1 \times n_1$ multiplied by a matrix $C$:
\begin{center}
\resizebox{0.5\textwidth}{!}{%
\includegraphics{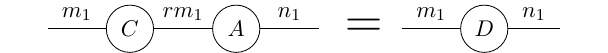}
}%
\end{center}
Then the solution is \[A=(C^+D).reshape(r, m_1, n_1).\] 
The pseudo-code is shown in Algorithm \ref{alg:weighted_kron}.

\begin{algorithm}
\caption{Kronecker appr. in weighted norm}\label{alg:weighted_kron}
\begin{algorithmic}
\State \textbf{Input}  $B\in \mathbb{R}^{r\times m_2\times n_2}, Y\in \mathbb{R}^{m_1m_2\times n_1n_2}, X \in \mathbb{R}^{m_1m_2\times m_1m_2}$.
\State \textbf{Solve} $\|Y-X\sum^{r}_{i}(A_i\otimes B_i)\|^2_F\to \min_A$
\State $Z = X^TX.\text{reshape}(m_1, m_2, m_1, m_2)$
\State $C = \text{einsum}(\text{'rij,nkj,aibk$\to$ranb'}, B, B, Z)$
\State $C=C.\text{reshape}(rm_1, rm_1)$
\State $T = X^TY.\text{reshape}(m_1, m_2, n_1, n_2)$
\State $D=\text{einsum}(\text{'rij,aibj$\to$rab'}, B, T)$
\State $D=D.\text{reshape}(rm_1, n_1)$
\State $A = (C^{+}D).\text{reshape}(r, m_1, n_1)$
\State \textbf{return} $A$ 
\end{algorithmic}
\end{algorithm}

\section{\texorpdfstring{$\mathcal{GS}$} approximation in weighted norm \label{sec:gs_weighted}}
Let $L$ be block-diagonal matrix with $kl$ blocks of size $L_i \in \mathbb{R}^{bl_1\times bl_2}$, $R$ be block-diagonal matrix with $kr$ blocks of size $R_i \in \mathbb{R}^{br_1\times br_2}$. Let us solve the approximation with $\mathcal{GS}$ matrices with fixed permutations $P_{L}, P, P_{R}$ in the  weighted norm:
\[\|Y'-X'P_LLPRP_R\|^2_F\to\min_{L, R}.\]
Due to the unitary invariance of Frobenius norm and orthogonality of permutation matrices, we can rewrite it as 
\[\|Y-XLPR\|^2_F\to\min_{L, R}, \]
where $X=X'P_L, Y=Y'P_R^T$. We can not solve this problem analytically, so we will optimize iteratively with respect to $L$ and $R$. 

To find optimal $R$ for the fixed $L$ we have to solve the least squares problem for each of the blocks $R_{i}$. Let us divide $Y$ and $XLP$ by columns into $kr$ blocks $Y_i \in \mathbb{R}^{kl\cdot bl1 \times bl2}$  and $(XLP)_i \in \mathbb{R}^{kl\cdot bl1 \times bl2}$. Then we can find the optimal $R$ as
\[R_i=(XLP)_i^+Y_i.\]
To solve the optimization problem with respect to $L$ for the fixed $R$, we use built-in iterative algorithm for least squares problem. 
However, to avoid poorly conditioned systems, we first divide $X$ and $PR$ into $kl$ blocks by columns and rows respectively, and make QR decomposition of $X_i$ and $(PR)_i$. 
The result is shown in Algorithm \ref{alg:monarch_L} below.

\begin{algorithm}
\caption{Find optimal $L$}\label{alg:monarch_L}
\begin{algorithmic}
\State \textbf{Input}  $X_i\in\mathbb{R}^{kl\cdot bl_1\times bl_1}, Y\in \mathbb{R}^{kl\cdot bl_1\times kr \cdot br_2}, (PR)_i \in \mathbb{R}^{kl\times kr\cdot br_2}$.
\State \textbf{Solve} $\|Y-\sum^{kl}_{i}X_iL_i(PR)_i\|^2_F\to \min_L$
\State $Q_{X_i}, R_{X_i}=qr(X_i)$
\State $Q_{(PR)^T_i}, R_{(PR)^T_i}=qr((PR)^T_i)$
\State $L = \argmin_L\left(Y-\sum_{i}^{kl}Q_{X_i}L_iQ^T_{(PR)^T_i}\right)$
\State $L_i=R_{X_i}^+L_i(R_{(PR)^T_i}^{T})^+$
\State \textbf{return} $L$ 
\end{algorithmic}
\end{algorithm}

\section{Proof of Proposition~\ref{prop:slicegpt}}
\label{apx:prop1_proof}
Firstly, let us write
\[\begin{split}
    &\|X_{out} (W_{out} Q - \widehat{W}_{out}) + \\& +X_{skip} (Q - Q_{\ell-1} \widehat{W}_{skip})\|_F^2 = \\
    &= \| (X_{out} W_{out} + X_{skip}) - \\ & - (X_{out}\widehat{W}_{out} + X_{skip}Q_{\ell-1} \widehat{W}_{skip})Q^T  \|_F^2\end{split}.
\]
Note that slicing of the columns of matrices $\widehat{W}_{skip}$ and $\widehat{W}_{out}$ make matrix $X_{out}\widehat{W}_{out} Q^T + X_{skip}  Q_{\ell - 1}\widehat{W}_{skip} Q^T$ at most of rank $d$, where $d$ is the dimension after slicing. 
If we now show, that we can choose appropriate matrices $Q, \widehat{W}_{out}, \widehat{W}_{skip}$ such that they correspond to best low-rank approximation of $X_{out} W_{out} + X_{skip}$, we will show that these matrices are the analytical solutions to our problem. Let $P$ equal to the following projector:
\[
    P = \begin{pmatrix}
        I_{d} & 0 \\
        0 & 0
    \end{pmatrix}
\]
Consider following matrices for $\widehat{W}_{out}, \widehat{W}_{skip}$:
\[
\widehat{W}_{out} = W_{out} Q P; \quad \widehat{W}_{skip} = Q_{\ell - 1}^T Q P
\] 
Substituting, we get:
\[
    \| (X_{out} W_{out} + X_{skip})Q (I - P) \|_F^2 \to \min_{Q^T Q = I}
\]
Solution to which is well-known and revolves around computing SVD of $(X_{out} W_{out} + X_{skip})$ (taking principal components). This also corresponds to the best low-rank approximation and therefore found solution is optimal. 

Up to a normalization layer, this corresponds to the rotations and slices of $W_{out}$ and $Q_{\ell - 1}^T Q_{\ell}$ layers applied in SliceGPT. Sparse block structure that arises in inputs after the first layers allows to also slice $W_{in}$ and $Q_{\ell - 1}^T Q_{\ell}$ along other dimension, making it equivalent to full SliceGPT scheme.

\section{Choice of sizes for factorizations.}
\label{sec:decomposition_sizes}
To compress a matrix $M$ of size $n\times m$ at approximately $p/q$ of its parameters with the sum of $r$ Kronecker products $\sum_i^r{A_i \otimes B_i}$, a natural choice is to set $r=p$, $A_i$ of size $q\times1$ and $B_i$ of size $n/q \times m$.  We assume $n$ is divisible by $q$. Then the number of parameters in the sum $\sum_i^{r=p}{A_i \otimes B_i}$ equals $p(q + (n/q)m) = qp + nm(p/q)$, which is approximately $p/q$ of initial matrix size. In our case, we use $q = 4$ and $r = 3$, yielding $\sim 25\%$ compression rate for layers. We use $q=5, r=8$ for $\sim 37.5\%$ compression, $q=1, r=2$ for $\sim 50\%$  compression. 

It is important to note, that splitting by $q$ should be applied on the side that is multiplied by an orthogonal matrix, so that it has an effect on approximation error. For example, if the matrix $W_{in}$ is multiplied by orthogonal matrix from the left, then decomposition should be $A_{i}\in \mathbb{R}^{q\times 1}, B_i\in \mathbb{R}^{n/q\times m}$ or $A_{i}\in \mathbb{R}^{n/q\times m}, B_i\in \mathbb{R}^{q\times 1}$. If $W_{out}$ is multiplied from the right, then $A_{i}\in \mathbb{R}^{1\times q}, B_i\in \mathbb{R}^{n\times m/q}$ or $A_{i}\in \mathbb{R}^{n\times m/q}, B_i\in \mathbb{R}^{1\times q}$. 

When compressing the matrix $M \in \mathbb{R}^{kl * bl_1 \times kr* br_2}$ with $GS$ decomposition using $L$ with $kl$ blocks of the size $bl_1\times bl_2$ and $R$ with $kr$ blocks of size $br_1\times br_2$, the compression ratio equals $c=\frac{kl*bl_1*bl_2+kr*br_1*br_2}{kl*kr*bl_1*br_2}$. From the definition of $GS$ matrices, $kl*bl_2=kr*br_1$, which means that it is easy to find $bl_2$ and $br_1$ if $c, kl, bl_1, kr, br_2$ are known. To compress square matrices at $c=3/4$, we choose $kl, kr = 4, 2$. For rectangular matrices, $kl, kr = 4, 8$. For embedding and head $kl, kr = 1, 4$.

\end{document}